\documentclass{article}

\usepackage{colt09e}
\usepackage{times}
\usepackage{amsfonts}
\usepackage{amsmath}
\usepackage[psamsfonts]{amssymb}
\usepackage{latexsym}
\usepackage{color}
\usepackage{graphics}
\usepackage{enumerate}
\usepackage{amstext}
\usepackage{url}
\usepackage{epsfig}
\usepackage{mlapa}

\def\Rset{\mathbb{R}}
\def\Sset{\mathbb{S}}

\def\dqed{\relax\tag*{\qed}}

\DeclareMathOperator*{\E}{\rm E}

\DeclareMathOperator*{\argmin}{\rm argmin}

\DeclareMathOperator{\supp}{supp}
\DeclareMathOperator{\diag}{diag}

\providecommand{\abs}[1]{\lvert#1\rvert}
\providecommand{\norm}[1]{\lVert#1\rVert}

\newtheorem{proposition}{Proposition}
\newcommand{\ipsfig}[2]{\scalebox{#1}{\psfig{#2}}}

\newcommand{\iprod}[2]{\left\langle #1 , #2 \right\rangle}
\newcommand{\set}[1]{\{#1\}}

\newcommand{\mat}[1]{{\mathbf #1}}
\newcommand{\dis}{\mathrm{disc}}
\renewcommand{\L}{{\cal L}}
\newcommand{\Q}{{\cal Q}}
\newcommand{\M}{\mat{M}}
\newcommand{\N}{\mat{N}}
\newcommand{\A}{\mat{A}}
\newcommand{\B}{\mat{B}}
\newcommand{\K}{\mat{K}}
\newcommand{\Y}{\mat{Y}}
\newcommand{\Z}{\mat{Z}}
\newcommand{\I}{\mat{I}}
\newcommand{\1}{\mat{1}}
\newcommand{\s}{\mat{s}}
\renewcommand{\u}{\mat{u}}
\newcommand{\x}{\mat{x}}
\newcommand{\w}{\mat{w}}
\newcommand{\z}{\mat{z}}
\renewcommand{\P}{\mat{\Phi}}
\newcommand{\R}{\mathfrak{R}}
\renewcommand{\S}{{\mathcal S}}
\newcommand{\T}{{\mathcal T}}
\newcommand{\tr}{\mathrm{{\bf tr}}}
\newcommand{\h}{\widehat}

\newcommand{\ggamma}{\boldsymbol{\gamma}}
\newenvironment{proof*}{\noindent{\bf Proof:}}{}
\newcommand{\ignore}[1]{}

\title{Domain Adaptation: Learning Bounds and Algorithms}

\author{Yishay Mansour\\
Google Research and
\\ Tel Aviv Univ.
\\\texttt{\small mansour@tau.ac.il}
\And Mehryar Mohri\\
Courant Institute and \\
Google Research \\
\texttt{\small mohri@cims.nyu.edu}
\And Afshin
Rostamizadeh\\
Courant Institute\\
New York University\\
\texttt{\small rostami@cs.nyu.edu}}

\begin{document}

\maketitle

\begin{abstract}
  This paper addresses the general problem of domain adaptation which
  arises in a variety of applications where the distribution of the
  labeled sample available somewhat differs from that of the test
  data.  Building on previous work by \emcite{bendavid}, we introduce
  a novel distance between distributions, \emph{discrepancy distance},
  that is tailored to adaptation problems with arbitrary loss
  functions. We give Rademacher complexity bounds for estimating the
  discrepancy distance from finite samples for different loss
  functions.  Using this distance, we derive novel generalization
  bounds for domain adaptation for a wide family of loss functions.
  We also present a series of novel adaptation bounds for large
  classes of regularization-based algorithms, including support vector
  machines and kernel ridge regression based on the empirical
  discrepancy. This motivates our analysis of the problem of
  minimizing the empirical discrepancy for various loss functions for
  which we also give novel algorithms. We report the results of
  preliminary experiments that demonstrate the benefits of our
  discrepancy minimization algorithms for domain adaptation.
\end{abstract}

\section{Introduction}

In the standard PAC model \cite{valiant} and other theoretical models
of learning, training and test instances are assumed to be drawn from
the same distribution. This is a natural assumption since, when the
training and test distributions substantially differ, there can be no
hope for generalization. However, in practice, there are several
crucial scenarios where the two distributions are more
similar and learning can be more effective. One such scenario
is that of \emph{domain adaptation}, the main topic of
our analysis.

The problem of domain adaptation arises in a variety of applications
in natural language processing
\cite{Dredze07Frustratingly,Blitzer07Biographies,jiang-zhai07,chelba,daume06},
speech processing
\cite{Legetter&Woodlang,Gauvain&Lee,DellaPietra,Rosenfeld96,jelinek,roark03supervised},
computer vision \cite{martinez}, and many other areas.  Quite often,
little or no labeled data is available from the \emph{target domain},
but labeled data from a \emph{source domain} somewhat similar to the
target as well as large amounts of unlabeled data from the target
domain are at one's disposal. The domain adaptation problem then
consists of leveraging the source labeled and target unlabeled data to
derive a hypothesis performing well on the target domain.

A number of different adaptation techniques have been introduced in
the past by the publications just mentioned and other similar work in
the context of specific applications. For example, a standard
technique used in statistical language modeling and other generative
models for part-of-speech tagging or parsing is based on the maximum a
posteriori adaptation which uses the source data as prior knowledge to
estimate the model parameters \cite{roark03supervised}.  Similar
techniques and other more refined ones have been used for training
maximum entropy models for language modeling or conditional models
\cite{DellaPietra,jelinek,chelba,daume06}.

The first theoretical analysis of the domain adaptation problem was
presented by \emcite{bendavid}, who gave VC-dimension-based
generalization bounds for adaptation in classification tasks. Perhaps,
the most significant contribution of this work was the definition and
application of a distance between distributions, the $d_A$ distance,
that is particularly relevant to the problem of domain adaptation and
that can be estimated from finite samples for a finite VC dimension,
as previously shown by \emcite{kifer}. This work was later extended by
\emcite{blitzer} who also gave a bound on the error rate of a
hypothesis derived from a weighted combination of the source data sets
for the specific case of empirical risk minimization.
A theoretical study of domain adaptation was presented by
\emcite{nips09}, where the analysis deals with the related but
distinct case of adaptation with multiple sources, and where the
target is a mixture of the source distributions.

This paper presents a novel theoretical and algorithmic analysis of
the problem of domain adaptation. It builds on the work of
\emcite{bendavid} and extends it in several ways. We introduce a novel
distance, the \emph{discrepancy distance}, that is tailored to
comparing distributions in adaptation. This distance coincides with
the $d_A$ distance for 0-1 classification, but it can be used to
compare distributions for more general tasks, including regression,
and with other loss functions. As already pointed out, a crucial
advantage of the $d_A$ distance is that it can be estimated from
finite samples when the set of regions used has finite
VC-dimension. We prove that the same holds for the discrepancy
distance and in fact give data-dependent versions of that statement
with sharper bounds based on the Rademacher complexity.

We give new generalization bounds for domain adaptation and point out
some of their benefits by comparing them with previous bounds.  We
further combine these with the properties of the discrepancy distance
to derive data-dependent Rademacher complexity learning bounds. We
also present a series of novel results for large classes of
regularization-based algorithms, including support vector machines
(SVMs) \cite{ccvv} and kernel ridge regression (KRR) \cite{krr}. We
compare the pointwise loss of the hypothesis returned by these
algorithms when trained on a sample drawn from the target domain
distribution, versus that of a hypothesis selected by these algorithms
when training on a sample drawn from the source distribution. We show
that the difference of these pointwise losses can be bounded by a term
that depends directly on the empirical discrepancy distance of the
source and target distributions.

These learning bounds motivate the idea of replacing the empirical
source distribution with another distribution with the same support
but with the smallest discrepancy with respect to the target empirical
distribution, which can be viewed as reweighting the loss on each
labeled point. We analyze the problem of determining the distribution
minimizing the discrepancy in both 0-1 classification and square loss
regression. We show how the problem can be cast as a linear program
(LP) for the 0-1 loss and derive a specific efficient combinatorial
algorithm to solve it in dimension one. We also give a polynomial-time
algorithm for solving this problem in the case of the square loss by
proving that it can be cast as a semi-definite program (SDP).
Finally, we report the results of preliminary experiments showing the
benefits of our analysis and discrepancy minimization algorithms.

In section~\ref{sec:prelim}, we describe the learning set-up for
domain adaptation and introduce the notation and Rademacher complexity
concepts needed for the presentation of our
results. Section~\ref{sec:dist} introduces the discrepancy distance
and analyzes its properties. Section~\ref{sec:bounds} presents our
generalization bounds and our theoretical guarantees for
regularization-based algorithms. Section~\ref{sec:alg} describes and
analyzes our discrepancy minimization
algorithms. Section~\ref{sec:exp} reports the results of our
preliminary experiments.

\section{Preliminaries}
\label{sec:prelim}

\subsection{Learning Set-Up}

We consider the familiar supervised learning setting where the
learning algorithm receives a sample of $m$ labeled points $\S = (z_1,
\ldots, z_m) = ((x_1, y_1), \ldots, (x_m, y_m)) \in (X \times Y)^m$,
where $X$ is the input space and $Y$ the label set, which is $\set{0,
  1}$ in classification and some measurable subset of $\Rset$ in
regression.

In the \emph{domain adaptation problem}, the training sample $\S$ is
drawn according to a \emph{source distribution} $Q$, while test points
are drawn according to a \emph{target distribution} $P$ that may
somewhat differ from $Q$. We denote by $f\colon X \to Y$ the target
labeling function. We shall also discuss cases where the source
labeling function $f_Q$ differs from the target domain labeling
function $f_P$. Clearly, this dissimilarity will need to be small for
adaptation to be possible.

We will assume that the learner is provided with an unlabeled sample
$\T$ drawn i.i.d.\ according to the target distribution $P$. We denote
by $L\colon Y \times Y \to \Rset$ a loss function defined over pairs
of labels and by $\L_Q(f, g)$ the expected loss for any two functions
$f, g\colon X \to Y$ and any distribution $Q$ over $X$: $\L_Q(f, g) =
\E_{x \sim Q} [L(f(x), g(x))]$.

The domain adaptation problem consists of selecting a hypothesis $h$
out of a hypothesis set $H$ with a small expected loss according to
the target distribution $P$, $\L_P(h, f)$.

\subsection{Rademacher Complexity}

Our generalization bounds will be based on the following
data-dependent measure of the complexity of a class of functions.

\begin{definition}[Rademacher Complexity]
  Let $H$ be a set of real-valued functions defined over a set
  $X$. Given a sample $S \!\in\! X^m$, the empirical Rademacher
  complexity of $H$ is defined as follows:
\begin{equation}
  \h \R_S(H) = \frac{2}{m} \E_\sigma \Big[\sup_{h \in H} \big|
    \sum_{i=1}^m \sigma_i h(x_i) \big| \, \Big| S = (x_1, \ldots, x_m) \Big].
\end{equation}
The expectation is taken over $\sigma = (\sigma_1, \ldots, \sigma_n)$
where $\sigma_i$s are independent uniform random variables taking
values in $\set{-1, +1}$. The Rademacher complexity of a hypothesis
set $H$ is defined as the expectation of $\h \R_S(H)$ over all samples
of size $m$:
\begin{equation}
	\R_m(H) = \E_S \big[ \h \R_S(H) \big| |S| =  m \big].
\end{equation}
\end{definition}
The Rademacher complexity measures the ability of a class of functions
to fit noise. The empirical Rademacher complexity has the added
advantage that it is data-dependent and can be measured from finite
samples. It can lead to tighter bounds than those based on other
measures of complexity such as the VC-dimension
\cite{koltchinskii_and_panchenko}.

We will denote by $\h R_S(h)$ the empirical average of a hypothesis $h
\colon X \to \Rset$ and by $R(h)$ its expectation over a sample $S$
drawn according to the distribution considered. The following is a
version of the Rademacher complexity bounds by
\emcite{koltchinskii_and_panchenko} and \emcite{bartlett}.  For
completeness, the full proof is given in the Appendix.

\begin{theorem}[Rademacher Bound]
\label{th:rademacher}
  Let $H$ be a class of functions mapping $Z = X \times Y$ to $[0, 1]$
  and $\S = (z_1, \ldots, z_m)$ a finite sample drawn i.i.d.\ according
  to a distribution $Q$. Then, for any $\delta > 0$, with probability
  at least $1 - \delta$ over samples $\S$ of size $m$, the following
  inequality holds for all $h \in H$:
\begin{equation}
R(h) \leq \h R(h) + \h \R_\S(H) + 3 \sqrt{\frac{\log \frac{2}{\delta}}{2m}}.
\end{equation}
\end{theorem}

\section{Distances between Distributions}
\label{sec:dist}

Clearly, for generalization to be possible, the distribution $Q$ and
$P$ must not be too dissimilar, thus some measure of the similarity of
these distributions will be critical in the derivation of our
generalization bounds or the design of our algorithms. This section
discusses this question and introduces a \emph{discrepancy}
distance relevant to the context of adaptation.

The $l_1$ distance yields a straightforward bound on the difference of
the error of a hypothesis $h$ with respect to $Q$ versus its error
with respect to $P$.

\begin{proposition}
\label{prop:l_1_bound}
Assume that the loss $L$ is bounded, $L \leq M$ for some $M > 0$.
Then, for any hypothesis $h \in H$,
\begin{equation}
\abs{\L_Q(h, f) - \L_P(h, f)} \leq M \, l_1(Q, P).
\end{equation}
\end{proposition}
\ignore{
\begin{proof*}
By definition of $\L_Q(h, f)$ and $\L_P(h, f)$, for any $h$,
\begin{multline*}
\abs{\L_Q(h, f) - \L_P(h, f)} \leq  \\
\sum_{x \in X} \abs{Q(x) - P(x)} \max_{x \in X} \abs{L(h(x), f(x))}.\dqed
\end{multline*}
\end{proof*}
}

This provides us with a first adaptation bound suggesting that for
small values of the $l_1$ distance between the source and target
distributions, the average loss of hypothesis $h$ tested on the target
domain is close to its average loss on the source domain.
However, in general, this bound is not informative since the $l_1$
distance can be large even in favorable adaptation situations.
Instead, one can use a distance between distributions better suited
to the learning task.

Consider for example the case of classification with the 0-1 loss. Fix
$h \in H$, and let $a$ denote the support of $\abs{h - f}$. Observe
that $\abs{\L_Q(h, f) - \L_P(h, f)} = \abs{Q(a) - P(a)}$.  A natural
distance between distributions in this context is thus one based on
the supremum of the right-hand side over all regions $a$. Since the
target hypothesis $f$ is not known, the region $a$ should be taken as
the support of $|h - h'|$ for any two $h, h' \in H$.

This leads us to the following definition of a distance originally
introduced by \emcite{devroye}~[pp. 271-272] under the name of
\emph{generalized Kolmogorov-Smirnov distance}, later by
\emcite{kifer} as \emph{the $d_A$ distance}, and introduced and
applied to the analysis of adaptation in classification by
\emcite{bendavid} and \emcite{blitzer}.

\begin{definition}[$d_A$-Distance]
  Let $A \subseteq 2^{\abs{X}}$ be a set of subsets of $X$. Then, the
  \emph{$d_A$-distance} between two distributions $Q_1$ and $Q_2$ over
  $X$, is defined as
\begin{equation}
d_A(Q_1,Q_2) = \sup_{a \in A} \abs{Q_1(a) - Q_2(a)}.
\end{equation}
\end{definition}

As just discussed, in 0-1 classification, a natural choice for $A$ is
$A = H \Delta H = \set{\abs{h' - h}\colon h, h' \in H}$. We introduce
a distance between distributions, \emph{discrepancy distance}, that
can be used to compare distributions for more general tasks, e.g.,
regression. Our choice of the terminology is partly motivated by the
relationship of this notion with the discrepancy problems arising in
combinatorial contexts \cite{chazelle}.

\begin{definition}[Discrepancy Distance] Let $H$ be a set of functions
  mapping $X$ to $Y$ and let $L\colon Y \times Y \to \Rset_+$ define a
  loss function over $Y$. The discrepancy distance $\dis_L$ between
  two distributions $Q_1$ and $Q_2$ over $X$ is defined by
\begin{multline*}
\dis_L(Q_1, Q_2) = \max_{h, h' \in H} \Big| \L_{Q_1}(h', h) 
- \L_{Q_2}(h', h) \Big|.
\end{multline*}
\end{definition}
The discrepancy distance is clearly symmetric and it is not hard to
verify that it verifies the triangle inequality, regardless of the
loss function used. In general, however, it does not define a
\emph{distance}: we may have $\dis_L(Q_1, Q_2) = 0$ for $Q_1 \neq
Q_2$, even for non-trivial hypothesis sets such as that of bounded
linear functions and standard continuous loss functions.

Note that for the 0-1 classification loss, the discrepancy distance
coincides with the $d_A$ distance with $A = H \Delta H$. But the
discrepancy distance helps us compare distributions for other losses
such as $L_q(y, y' ) = |y - y' |^q$ for some $q$ and is more general.

As shown by \emcite{kifer}, an important advantage of the $d_A$
distance is that it can be estimated from finite samples when $A$ has
finite VC-dimension. We prove that the same holds for the $\dis_L$
distance and in fact give data-dependent versions of that statement
with sharper bounds based on the Rademacher complexity.

The following theorem shows that for a bounded loss function $L$, the
discrepancy distance $\dis_L$ between a distribution and its empirical
distribution can be bounded in terms of the empirical Rademacher
complexity of the class of functions $L_H = \set{x \mapsto L(h'(x),
  h(x))\colon h, h' \in H}$. In particular, when $L_H$ has finite
pseudo-dimension, this implies that the discrepancy distance converges
to zero as $O(\sqrt{\log m/m})$.

\begin{proposition}
  Assume that the loss function $L$ is bounded by $M > 0$. Let $Q$ be
  a distribution over $X$ and let $\h Q$ denote the corresponding
  empirical distribution for a sample $\S = (x_1, \ldots ,
  x_m)$. Then, for any $\delta > 0$, with probability at least $1 -
  \delta$ over samples $\S$ of size $m$ drawn according to $Q$:
\begin{equation}
  \dis_L(Q, \h Q) \leq \h \R_\S(L_H) + 3 M \sqrt{\frac{\log \frac{2}{\delta}}{2m}}.
\end{equation}
\end{proposition}

\begin{proof}
  We scale the loss $L$ to $[0, 1]$ by dividing by $M$, and denote the
  new class by $L_H /M$.  By Theorem~\ref{th:rademacher} applied to
  $L_H /M$, for any $\delta > 0$, with probability at least $1 -
  \delta$, the following inequality holds for all $h, h' \in H$:
\begin{equation*}
  \frac{\L_Q(h', h)}{M} \leq \frac{\L_{\h Q}(h', h)}{M} +
 \h \R_\S(L_H/M) + 3  \sqrt{\frac{\log \frac{2}{\delta}}{2m}}.
\end{equation*}
The empirical Rademacher complexity has the property that $\h\R(\alpha
H) = \alpha \h\R(H)$ for any hypothesis class $H$ and positive real
number $\alpha$ \cite{bartlett}. Thus, $\R_\S(L_H/M)
= \frac{1}{M}\R_\S(L_H)$, which proves the proposition.
\end{proof}

For the specific case of $L_q$ regression losses, the bound can
be made more explicit.

\begin{corollary}
\label{cor:dis_bound}
Let $H$ be a hypothesis set bounded by some $M \!>\! 0$ for the loss
function $L_q$: $L_q(h, h') \leq M$, for all $h, h' \in H$. Let $Q$ be
a distribution over $X$ and let $\h Q$ denote the corresponding
empirical distribution for a sample $\S = (x_1, \ldots , x_m)$. Then,
for any $\delta > 0$, with probability at least $1 - \delta$ over
samples $\S$ of size $m$ drawn according to $Q$:
\begin{equation}
  \dis_{L_q}(Q, \h Q) \leq 4q \h\R_\S(H) + 3 M \sqrt{\frac{\log \frac{2}{\delta}}{2m}}.
\end{equation}
\end{corollary}
\begin{proof}
The function $f\colon x \mapsto x^q$ is $q$-Lipschitz for $x \in [0, 1]$:
\begin{equation}
|f(x') - f(x)| \leq q |x' - x|,
\end{equation}
and $f(0) = 0$. For $L = L_q$, $L_H = \set{x \mapsto \abs{h'(x) -
    h(x)}^q \colon h, h' \in H}$. Thus, by Talagrand's contraction
lemma \cite{talagrand}, $\h \R(L_H)$ is bounded by $2q \h \R(H')$ with
$H' \!=\! \set{x \mapsto (h'(x) - h(x))\colon h, h' \in H}$.
Then, $\h \R_\S(H')
$ can be written and bounded as follows
\begin{multline*}
\h \R_\S(H')  = \E_\sigma \bigl[
\sup_{h, h'} \frac{1}{m} |\sum_{i = 1}^m \sigma_i (h(x_i) - h'(x_i))|\bigr] \\
{\begin{aligned}
& \leq \E_\sigma [ \sup_{h} \frac{1}{m} |\sum_{i = 1}^m \sigma_i
h(x_i)| ] + \E_\sigma [ \sup_{h'} \frac{1}{m} |\sum_{i = 1}^m \sigma_i
h'(x_i)| ] \\
& = 2 \h \R_\S(H),
\end{aligned}}
\end{multline*}
using the definition of the Rademacher variables and the
sub-additivity of the supremum function. This proves the inequality
$\h\R(L_H) \leq 4q \h\R(h)$ and the corollary.
\end{proof}

A very similar proof gives the following result for classification.

\begin{corollary}
\label{cor:dis_class}
Let $H$ be a set of classifiers mapping $X$ to $\set{0, 1}$ and let
$L_{01}$ denote the 0-1 loss. Then, with the notation of
Corollary~\ref{cor:dis_bound}, for any $\delta > 0$, with probability
at least $1 - \delta$ over samples $\S$ of size $m$ drawn according to
$Q$:
\begin{equation}
  \dis_{L_{01}}(Q, \h Q) \leq 4 \h\R_\S(H) + 3 \sqrt{\frac{\log \frac{2}{\delta}}{2m}}.
\end{equation}
\end{corollary}
The factor of $4$ can in fact be reduced to $2$ in these corollaries
when using a more favorable constant in the contraction lemma. The
following corollary shows that the discrepancy distance can be
estimated from finite samples.
\begin{corollary}
\label{cor:dis_bound_emp}
Let $H$ be a hypothesis set bounded by some $M \!>\! 0$ for the loss
function $L_q$: $L_q(h, h') \!\leq\! M$, for all $h, h' \!\in\!  H$.
Let $Q$ be a distribution over $X$ and $\h Q$ the corresponding
empirical distribution for a sample $\S$, and let $P$ be a
distribution over $X$ and $\h P$ the corresponding empirical
distribution for a sample $\T$. Then, for any $\delta > 0$, with
probability at least $1 - \delta$ over samples $\S$ of size $m$ drawn
according to $Q$ and samples $\T$ of size $n$ drawn according to $P$:
\begin{multline*}
  \dis_{L_q}(P, Q) \leq \dis_{L_q}(\h P, \h Q) +\\
\qquad 4q \Big(\h\R_\S(H) + \h\R_\T(H)\Big) +
  3 M \Bigg(\sqrt{\frac{\log \frac{4}{\delta}}{2m}} + \sqrt{\frac{\log
      \frac{4}{\delta}}{2n}}\Bigg).
\end{multline*}
\end{corollary}
\begin{proof}
  By the triangle inequality, we can write
\begin{multline}
  \dis_{L_q}(P, Q) \leq \dis_{L_q}(P, \h P) + \dis_{L_q}(\h P, \h Q) +\\
\dis_{L_q}(Q, \h Q).
\end{multline}
The result then follows by the application of Corollary~\ref{cor:dis_bound}
to $\dis_{L_q}(P, \h P)$ and $\dis_{L_q}(Q, \h Q)$.
\end{proof}

As with Corollary~\ref{cor:dis_class}, a similar result holds for the 0-1 loss
in classification.

\section{Domain Adaptation: Generalization Bounds}
\label{sec:bounds}

This section presents generalization bounds for domain adaptation
given in terms of the discrepancy distance just defined. In the
context of adaptation, two types of questions arise:
\begin{enumerate}
\item[(1)] we may ask, as for standard generalization, how the average
  loss of a hypothesis on the target distribution, $\L_P(h, f)$,
  differs from $\L_{\h Q}(h, f)$, its empirical error based on the
  empirical distribution $\h Q$;

\item[(2)] another natural question is, given a specific learning
  algorithm, by how much does $\L_P(h_Q, f)$ deviate from $\L_P(h_P,
  f)$ where $h_Q$ is the hypothesis returned by the algorithm when
  trained on a sample drawn from $Q$ and $h_P$ the one it would have
  returned by training on a sample drawn from the true target
  distribution $P$.
\end{enumerate}
We will present theoretical guarantees addressing both questions.

\subsection{Generalization bounds}

Let $h_Q^* \in \argmin_{h \in H} \L_Q(h, f_Q)$ and similarly let
$h_P^*$ be a minimizer of $\L_P(h, f_P)$. Note that these minimizers
may not be unique. For adaptation to succeed, it is natural to assume
that the average loss $\L_Q(h_Q^*, h_P^*)$ between the best-in-class
hypotheses is small. Under that assumption and for a small discrepancy
distance, the following theorem provides a useful bound on the error
of a hypothesis with respect to the target domain.

\begin{theorem}
\label{th:gen_bound}
Assume that the loss function $L$ is symmetric and obeys the triangle
inequality. Then, for any hypothesis $h \in H$, the following holds
\begin{multline}
\L_P(h, f_P) \leq \L_P(h_P^*, f_P) + \L_Q(h, h_Q^*) + \dis_L(P, Q) \\
  + \min\{\L_Q(h_Q^*, h_P^*), \L_P(h_Q^*, h_P^*)\}.
\label{eq:gen_bound}
\end{multline}
\end{theorem}
\begin{proof*}
  We show two inequalities, the combination of which proves the
  theorem. Fix $h \in H$. By the triangle inequality property
  of $L$ and the definition of the discrepancy $\dis_L(P, Q)$,
  the following holds
\begin{align*}
\L_P(h, f_P)
& \leq \L_P(h, h_Q^*) + \L_P(h_Q^*, h_P^*) + \L_P(h_P^*, f_P)\\
& \leq \L_Q(h, h_Q^*) + \dis_L(P, Q) + \L_P(h_Q^*, h_P^*) \\
& \quad + \L_P(h_P^*, f_P).
\end{align*}
  Similarly, using same arguments, we have
\begin{align*}
\L_P(h, f_P)
& \leq \L_P(h, h_P^*) + \L_P(h_P^*, f_P)\\
& \leq \L_Q(h, h_P^*) + \dis_L(P, Q) + \L_P(h_P^*, f_P)\\
& \leq \L_Q(h, h_Q^*) + \L_Q(h_Q^*, h_P^*) + \dis_L(P, Q)  \\
& \quad  + \L_P(h_P^*, f_P). \dqed
\end{align*}
\end{proof*}
We compare (\ref{eq:gen_bound}) with the main adaptation bound given
by \emcite{bendavid} and \emcite{blitzer}:
\begin{multline}
  \L_P(h, f_P) \leq \L_Q(h, f_Q) + \dis_L(P, Q) + \\
\min_{h \in H} \big(\L_Q(h, f_Q) + \L_P(h, f_P) \big).
\label{bound-old-adap}
\end{multline}
It is very instructive to compare the two bounds.  Intuitively, the
bound of Theorem \ref{th:gen_bound} has only one error term that
involves the target function, while the bound of
(\ref{bound-old-adap}) has three terms involving the target function.
One extreme case is when there is a single hypothesis $h$ in $H$ and a
single target function $f$.  In this case, Theorem~\ref{th:gen_bound}
gives a bound of $\L_P(h, f) + \dis(P, Q) $, while the bound supplied
by (\ref{bound-old-adap}) is $2\L_Q(h, f) + \L_P(h, f) + \dis(P, Q)$,
which is larger than $3 \L_P(h, f) + \dis(P, Q)$ when $\L_Q(h, f) \leq
\L_P(h, f)$.  One can even see that the bound of
(\ref{bound-old-adap}) might become vacuous for moderate values of
$\L_Q(h, f)$ and $\L_P(h, f)$.  While this is clearly an extreme case,
an error with a factor of 3 can arise in more realistic situations,
especially when the distance between the target function and the
hypothesis class is significant.

While in general the two bounds are incomparable, it is worthwhile to
compare them using some relatively plausible assumptions.
Assume that the discrepancy distance between $P$ and $Q$ is small and
so is the average loss between $h_Q^*$ and $h_P^*$. These are natural
assumptions for adaptation to be possible. Then,
Theorem~\ref{th:gen_bound} indicates that the regret $\L_P(h, f_P) -
\L_P(h_P^*, f_P)$ is essentially bounded by $\L_Q(h, h_Q^*)$, the
average loss with respect to $h_Q^*$ on $Q$.  We now consider several
special cases of interest.

\begin{enumerate}
\item[(i)] When $h_Q^* = h_P^*$ then $h^* = h_Q^* = h_P^*$ and the bound of Theorem~\ref{th:gen_bound} becomes
\begin{equation}
\label{eq:bound1}
\L_P(h, f_P) \leq \L_P(h^*, f_P) + \L_Q(h, h^*) + \dis(P, Q).
\end{equation}
The bound of  (\ref{bound-old-adap}) becomes
\begin{multline*}
  \L_P(h, f_P) \leq \L_P(h^*, f_P) + \L_Q(h, f_Q) + \\\L_Q(h^*, f_Q) + \dis(P, Q),
\end{multline*}
where the right-hand side essentially includes the sum of $3$ errors
and is always larger than the right-hand side of (\ref{eq:bound1})
since by the triangle inequality $\L_Q(h, h^*) \leq \L_Q(h, f_Q)$ $ +
\L_Q(h^*, f_Q)$.
\item[(ii)] When $h_Q^* = h_P^*=h^* \wedge \dis(P, Q) = 0$, the bound of
  Theorem~\ref{th:gen_bound} becomes
\begin{equation*}
\L_P(h, f_P) \leq \L_P(h^*, f_P) + \L_Q(h, h^*),
\end{equation*}
which coincides with the standard generalization bound.
The bound of  (\ref{bound-old-adap}) does
not coincide with the standard bound and leads to:
\begin{equation*}
  \L_P(h, f_P) \leq \L_P(h^*, f_P) + \L_Q(h, f_Q) + \L_Q(h^*, f_Q).
\end{equation*}
\item[(iii)] When $f_P \!\in\! H$ (consistent case), the bound of (\ref{bound-old-adap}) simplifies to,
\begin{equation*}
\abs{\L_P(h, f_P) - \L_Q(h, f_P)} \leq \dis_L(Q, P),
\end{equation*}
and it can also be derived using the proof of Theorem \ref{th:gen_bound}.
\end{enumerate}
Finally, clearly Theorem~\ref{th:gen_bound} leads to bounds based on
the empirical error of $h$ on a sample drawn according to $Q$. We give
the bound related to the 0-1 loss, others can be derived in a similar
way from Corollaries~\ref{cor:dis_bound}-\ref{cor:dis_bound_emp} and
other similar corollaries. The result follows
Theorem~\ref{th:gen_bound} combined with
Corollary~\ref{cor:dis_bound_emp}, and a standard Rademacher
classification bound
(Theorem~\ref{th:rademacher_classification}) \cite{bartlett}.

\begin{theorem}
\label{th:gen_bound_emp}
Let $H$ be a family of functions mapping $X$ to $\set{0, 1}$ and let
the rest of the assumptions be as in
Corollary~\ref{cor:dis_bound_emp}. Then, for any hypothesis $h \in H$,
with probability at least $1 - \delta$, the following adaptation
generalization bound holds for the 0-1 loss:
\begin{multline}
\L_P(h, f_P) - \L_P(h_P^*, f_P) \leq \\
\L_{\h Q}(h, h_Q^*) +
\dis_{L_{01}}(\h P, \h Q) +
(4q + \frac{1}{2}) \h\R_\S(H) + 4 q \h\R_\T(H) + \\
  4 \sqrt{\frac{\log \frac{8}{\delta}}{2m}} + 3 \sqrt{\frac{\log
      \frac{8}{\delta}}{2n}} + \L_Q(h_Q^*, h_P^*).
\end{multline}
\end{theorem}

\subsection{Guarantees for regularization-based algorithms}

In this section, we first assume that the hypothesis set $H$ includes
the target function $f_P$. Note that this does not imply that $f_Q$ is
in $H$. Even when $f_P$ and $f_Q$ are restrictions to $\supp(P)$ and
$\supp(Q)$ of the same labeling function $f$, we may have $f_P \in H$
and $f_Q \not \in H$ and the source problem could be non-realizable.
Figure~\ref{fig:consistent} illustrates this situation.

\begin{figure}[t]
\begin{center}
\ipsfig{.4}{figure=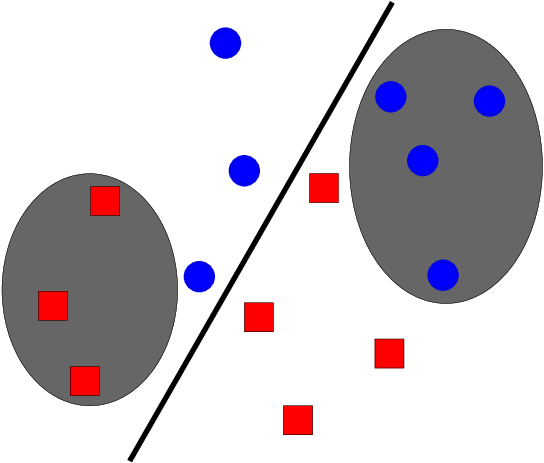}
\end{center}
\vspace{-.5cm}
\caption{In this example, the gray regions are assumed to have zero
  support in the target distribution $P$. Thus, there exist
  consistent hypotheses such as the linear separator
  displayed. However, for the source distribution $Q$ no linear
  separation is possible.}
\label{fig:consistent}
\vspace{-.5cm}
\end{figure}

For a fixed loss function $L$, we denote by $R_{\h Q}(h)$ the
empirical error of a hypothesis $h$ with respect to an empirical
distribution $\h Q$: $R_{\h Q}(h) = \L_{\h Q}(h, f)$.
Let $N\colon H \to \Rset_+$ be a function defined over the hypothesis
set $H$. We will assume that $H$ is a convex subset of a vector space
and that the loss function $L$ is convex with respect to each of its
arguments. Regularization-based algorithms minimize an objective of
the form
\begin{equation}
F_{\h Q}(h) = \h R_{\h Q}(h) + \lambda N(h),
\end{equation}
where $\lambda \geq 0$ is a trade-off parameter. This family of
algorithms includes support vector machines (SVM) \cite{ccvv}, support
vector regression (SVR) \cite{vapnik98}, kernel ridge regression
\cite{krr}, and other algorithms such as those based on the relative
entropy regularization \cite{bousquet-jmlr}.

We denote by $B_F$ the Bregman divergence associated to a convex
function $F$,
\begin{equation}
B_F(f \Arrowvert g) = F(f) - F(g) - \iprod{f - g}{\nabla
  F(g)}
\end{equation}
and define $\Delta h$ as $\Delta h = h' - h$.

\begin{lemma}
\label{lemma:stability}
Let the hypothesis set $H$ be a vector space. Assume that $N$ is a
proper closed convex function and that $N$ and $L$ are differentiable.
Assume that $F_{\h Q}$ admits a minimizer $h \in H$ and $F_{\h P}$ a
minimizer $h' \in H$ and that $f_P$ and $f_Q$ coincide on the support
of $\h Q$. Then, the following bound holds,
\begin{equation}
B_N(h' \Arrowvert h) + B_N(h \Arrowvert h') \leq \frac{2\dis_L(\h P, \h Q)}{\lambda}.
\end{equation}
\end{lemma}
\begin{proof}
Since $B_{F_{\h Q}} = B_{\h R_{\h Q}} + \lambda B_{N}$ and $B_{F_{{\h P}}} =
B_{\h R_{{\h P}}} + \lambda B_{N}$, and a Bregman divergence is
non-negative, the following inequality holds:
\begin{equation*}
\lambda \bigl(B_N(h' \Arrowvert h) + B_N(h \Arrowvert h')\bigr) \leq B_{F_{\h Q}}(h' \Arrowvert h) + B_{F_{{\h P}}}(h \Arrowvert h').
\end{equation*}
By the definition of $h$ and $h'$ as the minimizers of $F_{\h Q}$ and
$F_{{\h P}}$, $\nabla_{\h Q} F (h) = \nabla_{\h P} F (h') = 0$ and
\begin{multline*}
  \lambda \big(B_{F_{\h Q}}(h' \Arrowvert h) + B_{F_{{\h P}}}(h \Arrowvert h')\big) \\
{\begin{aligned}
& = \h R_{\h Q}(h') - \h R_{\h Q}(h) + \h R_{\h P}(h) - \h R_{\h P}(h')\\
& = \big(\L_{\h P}(h, f_P) - \L_{\h Q}(h, f_P)\big)\\
& \quad - \big(\L_{\h P}(h', f_P) - \L_{\h Q}(h', f_P)\big)
 \leq 2 \dis_L(\h P, \h Q).
\end{aligned}}
\end{multline*}
This last inequality holds since by assumption $f_P$ is in $H$.
\end{proof}

We will say that a loss function $L$ is \emph{$\sigma$-admissible}
when there exists $\sigma \in \Rset_+$ such that for any two
hypotheses $h, h' \in H$ and for all $x \in X$, and $y \in Y$,
\begin{equation}
\big|L(h(x), y) - L(h'(x), y)\big| \leq \sigma \big|h(x) - h'(x)\big|.
\end{equation}
This assumption holds for the hinge loss with $\sigma = 1$ and for the
$L_q$ loss with $\sigma = q (2M)^{q - 1}$ when the hypothesis set and the
set of output labels are bounded by some $M \in \Rset_+$: $\forall h
\in H, \forall x \in X, |h(x)| \leq M$ and $\forall y \in Y, |y| \leq
M$.

\begin{theorem}
\label{th:stability}
Let $K\colon X \times X \to \Rset$ be a positive-definite symmetric
kernel such that $K(x, x) \leq \kappa^2 < \infty$ for all $x \in X$,
and let $H$ be the reproducing kernel Hilbert space associated to
$K$. Assume that the loss function $L$ is $\sigma$-admissible. Let
$h'$ be the hypothesis returned by the regularization algorithm based
on $N(\cdot) = \norm{\cdot}_K^2$ for the empirical distribution $\h
P$, and $h$ the one returned for the empirical distribution $\h Q$,
and that and that $f_P$ and $f_Q$ coincide on $\supp(\h Q)$. Then, for
all $x \in X$, $y \in Y$,
\begin{equation}
  \big| L(h'(x), y) - L(h(x), y) \big| \leq \kappa \sigma \sqrt{\frac{\dis_L(\h P, \h Q)}{\lambda}}.
\end{equation}
\end{theorem}
\begin{proof}
  For $N(\cdot) = \norm{\cdot}_K^2$, $N$ is a proper closed convex
  function and is differentiable. We have $B_N(h' \Arrowvert h) =
  \norm{h' - h}_K^2$, thus $B_N(h' \Arrowvert h) + B_N(h \Arrowvert
  h') = 2 \norm{\Delta h}_K^2$. When $L$ is differentiable, by
  Lemma~\ref{lemma:stability},
\begin{equation}
\label{eq:Delta_h}
      2 \norm{\Delta h}_K^2 \leq \frac{2\dis_L(\h P, \h Q)}{\lambda}.
\end{equation}
This result can also be shown directly without assuming that $L$ is
differentiable by using the convexity of $N$ and the minimizing
properties of $h$ and $h'$ with a proof that is longer than that of
Lemma~\ref{lemma:stability}.

Now, by the reproducing property of $H$, for all $x \in H$, $\Delta
h(x) = \iprod{\Delta h}{K(x, \cdot)}$ and by the Cauchy-Schwarz
inequality, $\abs{\Delta h(x)} \leq \norm{\Delta h}_K (K(x, x))^{1/2}
\leq \kappa \norm{\Delta h}_K$. By the $\sigma$-admissibility of $L$,
for all $x \in X$, $y \in Y$,
\begin{equation*}
\abs{L(h'(x), y) - L(h(x), y)} \leq \sigma \abs{\Delta h (x)} \leq \kappa \sigma \norm{\Delta h}_K,
\end{equation*}
which, combined with (\ref{eq:Delta_h}), proves the statement of the
theorem.
\end{proof}

Theorem~\ref{th:stability} provides a guarantee on the pointwise
difference of the loss for $h'$ and $h$ with probability one, which of
course is stronger than a bound on the difference between expected
losses or a probabilistic statement. The result, as well as the proof,
also suggests that the discrepancy distance is the ``right'' measure
of difference of distributions for this context. The theorem applies
to a variety of algorithms, in particular SVMs combined with arbitrary
PDS kernels and kernel ridge regression.

In general, the functions $f_P$ and $f_Q$ may not coincide on
$\supp(\h Q)$. For adaptation to be possible, it is reasonable to
assume however that
\begin{equation*}
  L_{\h Q}(f_Q(x), f_P(x)) \ll 1 \quad \text{and}
  \quad L_{\h P}(f_Q(x), f_P(x)) \ll 1.
\end{equation*}
This can be viewed as a condition on the proximity of the labeling
functions (the $Y$s), while the discrepancy distance relates to the
distributions on the input space (the $X$s). The following result
generalizes Theorem~\ref{th:stability} to this setting in the case of
the square loss.

\begin{theorem}
\label{th:stability2}
Under the assumptions of Theorem~\ref{th:stability}, but with $f_Q$
and $f_P$ potentially different on $\supp(\h Q)$, when $L$ is the
square loss $L_2$ and $\delta^2 = L_{\h Q}(f_Q(x), f_P(x)) \ll 1$, then,
for all $x \in X$, $y \in Y$,
\begin{multline}
  \big| L(h'(x), y) - L(h(x), y) \big| \leq \\  \frac{2 \kappa M}{\lambda}\Big(\kappa \delta +
\sqrt{\kappa^2 \delta^2 + 4 \lambda \dis_L(\h P, \h Q)}\Big).
\end{multline}
\end{theorem}
\begin{proof}
Proceeding as in the proof of Lemma~\ref{lemma:stability} and using the
definition of the square loss and the Cauchy-Schwarz inequality give
\begin{multline*}
  \lambda \big(B_{F_{\h Q}}(h' \Arrowvert h) + B_{F_{{\h P}}}(h \Arrowvert h')\big) \\
{\begin{aligned}
& = \h R_{\h Q}(h') - \h R_{\h Q}(h) + \h R_{\h P}(h) - \h R_{\h P}(h')\\
& = \big(\L_{\h P}(h, f_P) - \L_{\h Q}(h, f_P)\big)\\
& \quad - \big(\L_{\h P}(h', f_P) - \L_{\h Q}(h', f_P)\big)\\
& \qquad + 2 \E_{\h Q} [(h'(x) - h(x)) (f_P(x) - f_Q(x)]\\
& \leq 2 \dis_L(\h P, \h Q) + 2 \sqrt{\E_{\h Q} [\Delta h(x)^2] \E_{\h Q}[L(f_P(x), f_Q(x))]}\\
& \leq 2 \dis_L(\h P, \h Q) + 2 \kappa \norm{\Delta h}_K \delta.
\end{aligned}}
\end{multline*}
Since $N(\cdot) = \norm{\cdot}_K^2$, the inequality can be rewritten
as
\begin{equation}
\lambda \norm{\Delta h}_K^2 \leq \dis_L(\h P, \h Q) + \kappa \delta \norm{\Delta h}_K .
\end{equation}
Solving the second-degree polynomial in $\norm{\Delta h}_K$ leads to
the equivalent constraint
\begin{equation}
\norm{\Delta h}_K \leq \frac{1}{2 \lambda}\Big(\kappa \delta + \\
\sqrt{\kappa^2 \delta^2 + 4 \lambda \dis_L(\h P, \h Q)}\Big).
\end{equation}
The result then follows by the $\sigma$-admissibility of $L$ as in
the proof of Theorem~\ref{th:stability}, with $\sigma = 4M$.
\end{proof}

Using the same proof schema, similar bounds can be derived for other
loss functions.

When the assumption $f_P \in H$ is relaxed, the following theorem
holds.

\begin{theorem}
\label{th:stability3}
Under the assumptions of Theorem~\ref{th:stability}, but with $f_P$
not necessarily in $H$ and $f_Q$ and $f_P$ potentially different on
$\supp(\h Q)$, when $L$ is the square loss $L_2$ and $\delta' = L_{\h
  Q}(h_P^*(x), f_Q(x))^{1/2}  + L_{\h
  P}(h_P^*(x), f_P(x))^{1/2} \ll 1$, then, for all $x \in X$, $y \in Y$,
\begin{multline}
  \big| L(h'(x), y) - L(h(x), y) \big| \leq \\  \frac{2 \kappa M}{\lambda}\Big(\kappa \delta' +
\sqrt{\kappa^2 \delta'^2 + 4 \lambda \dis_L(\h P, \h Q)}\Big).
\end{multline}
\end{theorem}
\begin{proof}
Proceeding as in the proof of Theorem~\ref{th:stability2} and using the
definition of the square loss and the Cauchy-Schwarz inequality give
\begin{multline*}
  \lambda \big(B_{F_{\h Q}}(h' \Arrowvert h) + B_{F_{{\h P}}}(h \Arrowvert h')\big) \\
{\begin{aligned}
& = \big(\L_{\h P}(h, h_P^*) - \L_{\h Q}(h, h_P^*)\big)\\
& \quad - \big(\L_{\h P}(h', h_P^*) - \L_{\h Q}(h', h_P^*)\big)\\
& \qquad - 2 \E_{\h P} [(h'(x) - h(x)) (h_P^*(x) - f_P(x)]\\
& \qquad + 2 \E_{\h Q} [(h'(x) - h(x)) (h_P^*(x) - f_Q(x)]\\
& \leq 2 \dis_L(\h P, \h Q) + 2 \sqrt{\E_{\h P} [\Delta h(x)^2] \E_{\h P}[L(h_P^*(x), f_P(x))]} \\
& + 2 \sqrt{\E_{\h Q} [\Delta h(x)^2] \E_{\h Q}[L(h_P^*(x), f_Q(x))]}\\
& \leq 2 \dis_L(\h P, \h Q) + 2 \kappa \norm{\Delta h}_K \delta'.
\end{aligned}}
\end{multline*}
The rest of the proof is identical to that of Theorem~\ref{th:stability2}.
\end{proof}

\section{Discrepancy Minimization Algorithms}
\label{sec:alg}

The discrepancy distance $\dis_L(\h P, \h Q)$ appeared as a critical
term in several of the bounds in the last section. In particular,
Theorems~\ref{th:stability} and \ref{th:stability2} suggest that if we
could select, instead of $\h Q$, some other empirical distribution $\h
Q'$ with a smaller empirical discrepancy $\dis_L(\h P, \h Q')$ and use
that for training a regularization-based algorithm, a better guarantee
would be obtained on the difference of pointwise loss between $h'$ and
$h$. Since $h'$ is fixed, a sufficiently smaller discrepancy would
actually lead to a hypothesis $h$ with pointwise loss closer to that
of $h'$.

The training sample is given and we do not have any control over the
support of $\h Q$. But, we can search for the distribution $\h Q'$
with the minimal empirical discrepancy distance:
\begin{equation}
\label{eq:dis_min}
\h Q' = \argmin_{\h Q' \in \Q} \dis_L(\h P, \h Q'),
\end{equation}
where $\Q$ denotes the set of distributions with support $\supp(\h
Q)$. This leads to an optimization problem that we shall study
in detail in the case of several loss functions.

Note that using $\h Q'$ instead of $\h Q$ for training can be viewed
as \emph{reweighting} the cost of an error on each training point.
The distribution $\h Q'$ can be used to emphasize some points or
de-emphasize others to reduce the empirical discrepancy distance.
This bears some similarity with the reweighting or
\emph{importance weighting} ideas used in statistics and machine
learning for sample bias correction techniques \cite{elkan,bias} and other
purposes. Of course, the objective optimized here based on the
discrepancy distance is distinct from that of previous reweighting
techniques.

We will denote by $S_Q$ the support of $\h Q$, by $S_P$ the support of
$\h P$, and by $S$ their union $\supp(\h Q) \cup \supp(\h P)$, with
$|S_Q| = m_0 \leq m$ and $|S_P| = n_0 \leq n$.

In view of the definition of the discrepancy distance, problem
(\ref{eq:dis_min}) can be written as a min-max problem:
\begin{equation}
\h Q' = \argmin_{\h Q' \in \Q} \max_{h, h' \in H}
\abs{\L_{\h P}(h', h) - \L_{\h Q'}(h', h)}.
\end{equation}
As with all min-max problems, the problem has a natural game
theoretical interpretation. However, here, in general, we cannot
permute the $\min$ and $\max$ operators since the convexity-type
assumptions of the minimax theorems do not hold. Nevertheless, since
the max-min value is always a lower bound for the min-max, it provides
us with a lower bound on the value of the game, that is the minimal
discrepancy:
\begin{multline}
\label{eq:max-min}
  \max_{h, h' \in H} \min_{\h Q' \in \Q} \abs{\L_{\h P}(h', h) -
    \L_{\h Q'}(h', h)}
  \leq\\
  \min_{\h Q' \in \Q} \max_{h, h' \in H} \abs{\L_{\h P}(h', h) -
    \L_{\h Q'}(h', h)}.
\end{multline}
We will later make use of this inequality. Let us now examine the
minimization problem (\ref{eq:dis_min}) and its algorithmic solutions
in the case of classification with the 0-1 loss and regression with
the $L_2$ loss.

\subsection{Classification, 0-1 Loss}

For the 0-1 loss, the problem of finding the best distribution $\h Q'$
can be reformulated as the following min-max program:
\begin{align}
    & \min_{\ Q'} \max_{a \in H \Delta H} \big| \h Q'(a) -
      \h P(a) \big| \\
    & \text{subject to} \quad \forall x \in S_Q, \h Q'(x) \geq 0 \wedge \sum_{x \in S_Q} \h Q'(x) = 1,
\end{align}
where we have identified $H \Delta H = \set{\abs{h' - h}\colon h, h'
  \in H}$ with the set of regions $a \subseteq X$ that are the support
of an element of $H \Delta H$. This problem is similar to the min-max
resource allocation problem that arises in task optimization
\cite{kouvelis}. It can be rewritten as the following linear program
(LP):
\begin{align}
    \min_{\ Q'} & \quad \delta \\
    \text{subject to} & \quad \forall a \in H \Delta H, \h Q'(a) -
      \h P(a)  \leq \delta\\
       & \quad \forall a \in H \Delta H,  \h P(a) - \h Q'(a)
        \leq \delta\\
      & \quad \forall x \in S_Q, \h Q'(x) \geq 0 \wedge \sum_{x \in S_Q}
      \h Q'(x) = 1.
\end{align}
The number of constraints is proportional to $|H \Delta H|$ but it can be
reduced to a finite number by observing that two subsets $a, a'
\!\in\! H \Delta H$ containing the same elements of $S$ lead to
redundant constraints, since
\begin{equation}
\big| \h Q'(a) - \h P(a) \big| = \big| \h Q'(a') - \h P(a') \big|.
\end{equation}
Thus, it suffices to keep one canonical member $a$ for each such
equivalence class. The necessary number of constraints to be
considered is proportional to $\Pi_{H \Delta H}(m_0 + n_0)$, the
shattering coefficient of order $(m_0 + n_0)$ of the hypothesis class
$H \Delta H$. By the Sauer's lemma, this is bounded in terms of the
VC-dimension of the class $H \Delta H$, $\Pi_{H \Delta H}(m_0 + n_0)
\leq O((m_0 + n_0)^{VC(H \Delta H)})$, which can be bounded by $O((m_0
+ n_0)^{2 VC(H)})$ since it is not hard to see that $VC(H \Delta H)
\leq 2 VC(H)$.

In cases where we can test efficiently whether there exists a
consistent hypothesis in $H$, e.g., for half-spaces in
$\mathbb{R}^d$, we can generate in time $O((m_0+n_0)^{2d})$ all
consistent labeling of the sample points by $H$.  
(We remark that computing the discrepancy with the 0-1 loss is closely
related to agnostic learning. The implications of this fact will be
described in a longer version of this paper.)

\ignore{ we can
  enumerate all canonical sets $a \in H \Delta H$ by considering all
  combinations of $ {m \choose 2 d} = O( m^{2d})$ points which define
  pairs of half-spaces.  More generally, if for a hypothesis set $H$
  we have access to an algorithm that can test in polynomial time if
  there exists $h \in H$ that is consistent with the a labeling, we
  can determine all canonical sets in $O(m^d)$.  }

\begin{figure}[t]
\begin{center}
\ipsfig{.5}{figure=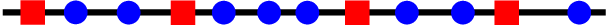}\\
(a)\\[.33cm]
\ipsfig{.5}{figure=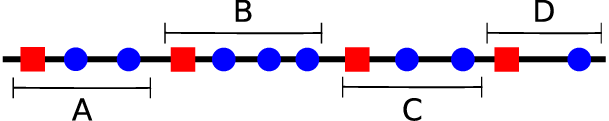}\\
(b)
\end{center}
\vspace{-.5cm}
\caption{Illustration of the discrepancy minimization algorithm in
  dimension one. (a) Sequence of labeled (red) and unlabeled (blue)
  points. (b) The weight assigned to each labeled point is the sum of
  the weights of the consecutive blue points on its right.}
\label{fig:1d_example}
\vspace{-.5cm}
\end{figure}

\subsection{Computing the Discrepancy in 1D}

We consider the case where $X = [0, 1]$ and derive a simple algorithm
for minimizing the discrepancy for 0-1 loss.  Let $H$ be the class of
all prefixes (i.e., $[0,z]$) and suffixes (i.e., $[z,1]$).  Our class
of $H\Delta H$ includes all the intervals (i.e., $(z_1,z_2]$) and
their complements (i.e., $[0,z_1]\cup (z_2,1]$).  We start with a
general lower bound on the discrepancy.

Let $U$ denote the set of \emph{unlabeled regions}, that is the set of
regions $a$ such that $a \cap S_Q = \emptyset$ and $a \cap S_P \neq
\emptyset$.
\ignore{ For a fixed region $a$ containing at least one point in
$S_Q$, $\min_{\h Q'} \abs{\h Q'(a) - \h P(a)} = 0$, since we can
define $\h Q'$ by assigning the distribution mass $\h P(a)$ to that
point and no mass to other labeled points in that region. But,}
If $a$ is an unlabeled region, then $\abs{\h Q'(a) - \h P(a)} =
\h P(a)$ for any $\h Q'$. Thus, by the max-min inequality (\ref{eq:max-min}), the
following lower bound holds for the minimum discrepancy:
\begin{equation}
\label{eq:lower_bound}
  \max_{a \in U} \h P(a) \leq
  \min_{\h Q' \in \Q} \max_{h, h' \in H} \abs{\L_{\h P}(h', h) -
    \L_{\h Q'}(h', h)}.
\end{equation}
In particular, if there is a large unlabeled region $a$, we cannot
hope to achieve a small empirical discrepancy.

In the one-dimensional case, we give a simple linear-time algorithm
that does not require an LP and show that the lower bound
(\ref{eq:lower_bound}) is reached. Thus, in that case, the $\min$ and
$\max$ operators commute and the minimal discrepancy distance is
precisely $\min_{a \in U} \h P(a)$.

Given our definition of $H$, the unlabeled regions are open intervals, or complements of these
sets, containing only points from $S_P$ with endpoints defined by
elements of $S_Q$.

Let us denote by $s_1, \ldots, s_{m_0}$ the elements of $S_Q$, by
$n_i$, $i \in [1, m_0]$, the number of consecutive unlabeled points to
the right of $s_i$ and $n=\sum n_i$. We will make an additional
technical assumption that there are no unlabeled points to the left of
$s_1$. Our algorithm consists of defining the weight $\h Q'(s_i)$ as
follows:
\begin{equation}
\h Q'(s_i) = n_i/n.
\end{equation}
This requires first sorting $S_Q\cup S_P$ and then computing $n_i$ for
each $s_i$. Figure~\ref{fig:1d_example} illustrates the algorithm.

\begin{proposition}
\label{prop:1d_alg}
Assume that $X$ consists of the set of points on the real line and $H$
the set of half-spaces on $X$. Then, for any $\h Q$ and $\h P$, $\h
Q'(s_i) = n_i/n$ minimizes the empirical discrepancy and can be
computed in time $O((m + n) \log (m + n))$.
\end{proposition}
\begin{proof}
  Consider an interval $[z_1, z_2]$ that maximizes the discrepancy of
  $\h Q'$. The case of a complement of an interval is the same, since
  the discrepancy of a hypothesis and its negation are identical. Let
  $s_i, \ldots, s_j \in [z_1, z_2]$ be the subset of $\h Q$ in that
  interval, and $p_{i'}, \ldots, p_{j'}\in [z_1, z_2]$ be the subset of
  $\h P$ in that interval.  The discrepancy is $d=|\sum_{k=i}^j \h
  Q'(s_k) - \frac{j'-i'}{n}|$.  By our definition of $\h Q'$, we have
  that $\sum_{k=i}^j \h Q'(s_k)= \frac{1}{n}\sum_{k=i}^j n_k$.  Let
  $p_{i''}$ be the maximal point in $\h P$ which is less than $s_i$
  and $j''$ the minimal point in $\h P$ larger than $s_j$.  We have
  that $j'-i' = (i''-i') + \sum_{k=i}^{j-1} n_k + (j''-j'))$.
  Therefore $d=| (i''-i') +(j''-j')-n_j|= | (i''-i') - (n_j
  -(j''-j'))|$.  Since $d$ is maximal and both terms are
  non-negative, one of them is zero.  Since $j'-j'' \leq n_j$ and
  $i''-i'\leq n_i$, the discrepancy of $\h Q'$ meets the lower bound
  of (\ref{eq:lower_bound}) and is thus optimal.
\end{proof}

\subsection{Regression, $L_2$ loss}
\label{sec:l2}

For the square loss, the problem of finding the best distribution
can be written as
\begin{align*}
    & \min_{\h Q' \in \Q} \max_{h, h' \in H} \Big| \E_{\h P}[(h'(x) - h(x))^2] - \E_{\h Q'}[(h'(x) - h(x))^2] \Big|.
\end{align*}
If $X$ is a subset of $\Rset^N$, $N \!>\! 1$, and the hypothesis set
$H$ is a set of bounded linear functions $H = \set{\x \mapsto \w^\top
  \x\colon \norm{\w} \!\leq\! 1}$, then, the problem can be rewritten
as
\begin{align}
& \min_{\h Q' \in \Q} \max_{\substack{\norm{\w} \leq 1\\\norm{\w'} \leq 1}} \Big| \E_{\h P}[((\w' - \w)^\top \x)^2] - \E_{\h Q'}[((\w' - \w)^\top \x)^2]\Big| \nonumber\\
& = \min_{\h Q' \in \Q} \max_{\substack{\norm{\w} \leq 1\\\norm{\w'} \leq 1}} \Big| \sum_{\x \in S} (\h P(\x) - \h Q'(\x))[(\w' - \w)^\top \x]^2 \Big| \nonumber\\
& = \min_{\h Q' \in \Q} \max_{\norm{\u} \leq 2} \Big| \sum_{\x \in S} (\h P(\x) - \h Q'(\x))[\u^\top \x]^2 \Big| \nonumber\\
\label{eq:34}
& = \min_{\h Q' \in \Q} \max_{\norm{\u} \leq 2} \Big| \u^\top \big(\sum_{\x \in S} (\h P(\x) - \h Q'(\x)) \x \x^\top\big) \u \Big|.
\end{align}
We now simplify the notation and denote by $\s_1, \ldots, \s_{m_0}$ the
elements of $S_Q$, by $z_i$ the distribution weight at point $\s_i$:
$z_i = \h Q'(\s_i)$, and by $\M(\z) \in \Sset^N$ a symmetric matrix
that is an affine function of $\z$:
\begin{equation}
\M(\z) = \M_0 - \sum_{i = 1}^{m_0} z_i \M_i,
\end{equation}
where $\M_0 = \sum_{\x \in S} P(\x) \x\x^\top$ and $\M_i =
\s_i \s_i^\top$. Since problem (\ref{eq:34}) is invariant to the non-zero
bound on $\norm{\u}$, we can equivalently write it with a bound of one
and in view of the notation just introduced give its equivalent form
\begin{equation}
\label{eq:36}
\min_{\substack{\| \z \|_1 = 1\\ \z \geq 0}} \max_{\norm{\u} = 1} \abs{\u^\top \M(\z) \u}.
\end{equation}
Since $\M(\z)$ is symmetric, $\max_{\norm{\u} = 1} \u^\top \M(\z) \u$
is the maximum eigenvalue $\lambda_{\max}$ of $\M(\z)$ and the problem
is equivalent to the following maximum eigenvalue minimization for
a symmetric matrix:
\begin{equation}
\min_{\substack{\| \z \|_1 = 1\\ \z \geq 0}} \max\set{\lambda_{\max}(\M(\z)), \lambda_{\max}(-\M(\z))},
\end{equation}
This is a convex optimization problem since the maximum eigenvalue of
a matrix is a convex function of that matrix and $\M$ is an affine
function of $\z$, and since $\z$ belongs to a simplex.
The problem is equivalent to the following semi-definite programming (SDP) problem:
\begin{align}
\label{eq:sdp}
\min_{\z, \lambda} & \quad \lambda\\
\text{subject to} & \quad \lambda \I - \M(\z) \succeq 0\\
& \quad \lambda \I + \M(\z) \succeq 0\\
& \quad \1^\top \z = 1 \wedge \z \geq 0.
\end{align}
SDP problems can be solved in polynomial time using general interior
point methods \cite{nesterov}. Thus, using the general expression of
the complexity of interior point methods for SDPs, the following
result holds.

\begin{proposition}
\label{prop:sdp}
  Assume that $X$ is a subset of $\Rset^N$ and that the hypothesis set
  $H$ is a set of bounded linear functions $H = \set{\x \mapsto
    \w^\top \x\colon \norm{\w} \!\leq\! 1}$. Then, for any $\h Q$
  and $\h P$, the discrepancy minimizing distribution $\h Q'$ for the
  square loss can be found in time $O(m_0^2 N^{2.5} + n_0N^2)$.

\end{proposition}
It is worth noting that the unconstrained version of this problem (no
constraint on $\z$) and other close problems seem to have been studied
by a number of optimization publications
\cite{fletcher,overton,jarre,helmberg,alizadeh}. This suggests
possibly more efficient specific algorithms than general interior
point methods for solving this problem in the constrained case as
well. Observe also that the matrices $\M_i$ have a specific structure
in our case, they are rank-one matrices and in many applications quite
sparse, which could be further exploited to improve efficiency.

\ignore{
The results of this section can be extended to the case where 
PDS kernels are used (see Appendix C).
}

\section{Experiments}
\label{sec:exp}

This section reports the results of preliminary experiments showing
the benefits of our discrepancy minimization algorithms.  Our results
confirm that our algorithm is effective in practice and produces a
distribution that reduces the empirical discrepancy distance, which
allows us to train on a sample closer to the target distribution with
respect to this metric. They also demonstrate the accuracy benefits of
this algorithm with respect to the target domain.

\begin{figure}[t]
\begin{center}
\begin{tabular}{cc}
\ipsfig{.45}{figure=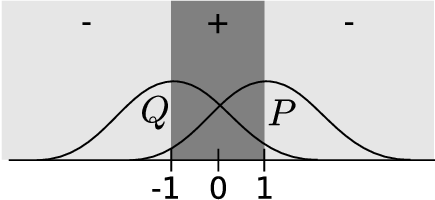}
& \ipsfig{.125}{figure=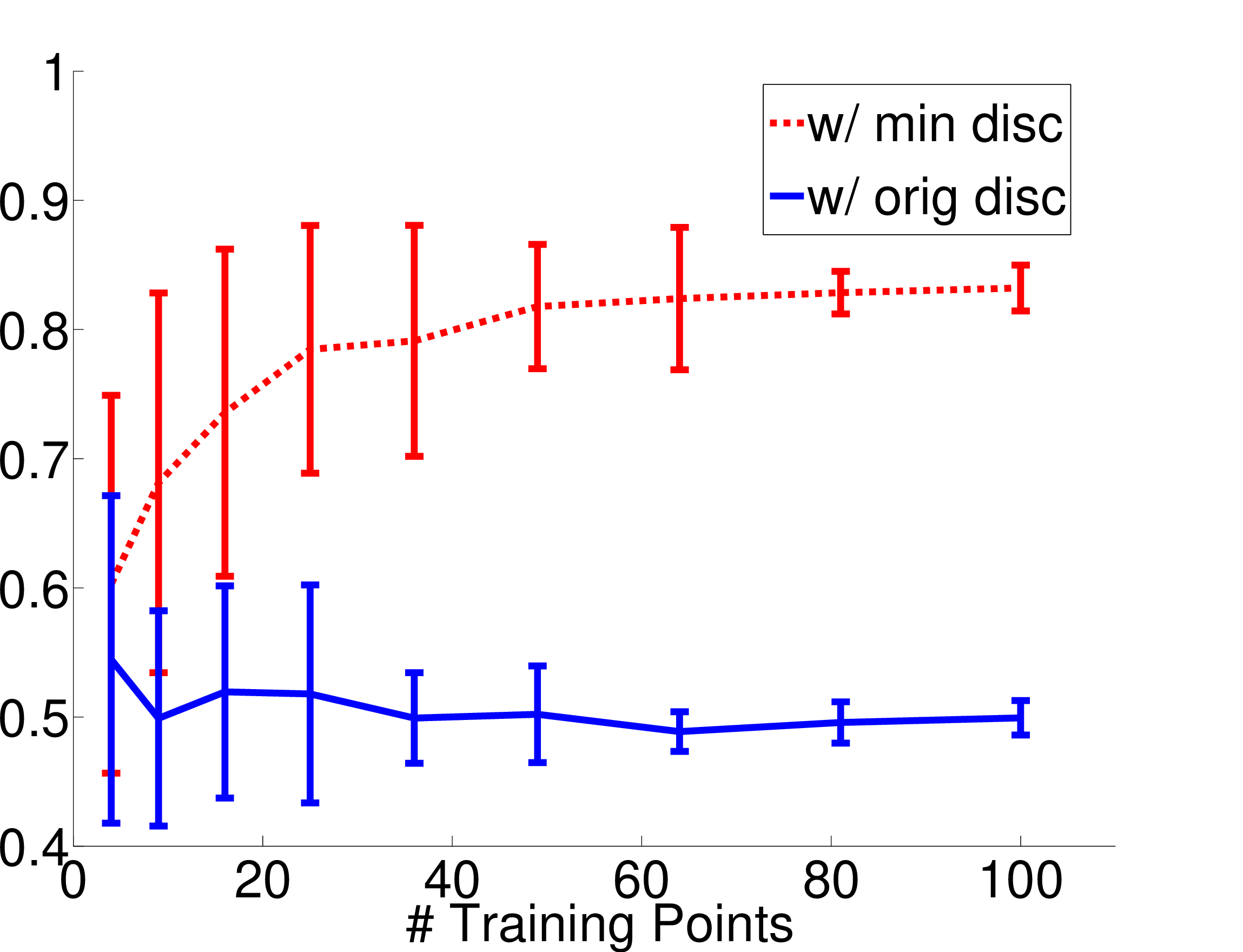}\\
(a) & (b)
\end{tabular}
\end{center}
\vspace{-.5cm}
\caption{Example of application of the discrepancy minimization
  algorithm in dimensions one. (a) Source and target distributions $Q$
  and $P$.  (b) Classification accuracy empirical results plotted as a function
of the number of training points for both the \emph{unweighted case}
(using original empirical distribution $\h Q$) and the \emph{weighted case}
(using distribution $\h Q'$ returned by our discrepancy minimizing algorithm).
The number of unlabeled points used was ten times the number of labeled.
Error bars show $\pm1$ standard deviation.}
\label{fig:ex_illustration}
\vspace{-.5cm}
\end{figure}

Figures~\ref{fig:ex_illustration}(a)-(b) show the empirical advantages
of using the distribution $\h Q'$ returned by the discrepancy
minimizing algorithm described in Proposition \ref{prop:1d_alg} in a
case where source and target distributions are shifted Gaussians: the
source distribution is a Gaussian centered at $-1$ and the target
distribution a Gaussian centered at $+1$, both with standard deviation
2.  The hypothesis set used was the set of half-spaces and the target
function selected to be the interval $[-1,1]$. Thus, training on a
sample drawn form $Q$ generates a separator at $-1$ and errs on about
half of the test points produced by $P$.  In contrast, training with
the distribution $\h Q'$ minimizing the empirical discrepancy yields a
hypothesis separating the points at $+1$, thereby dramatically
reducing the error rate.

\begin{figure}[t]
\begin{center}
\begin{tabular}{cc}
\ipsfig{.11}{figure=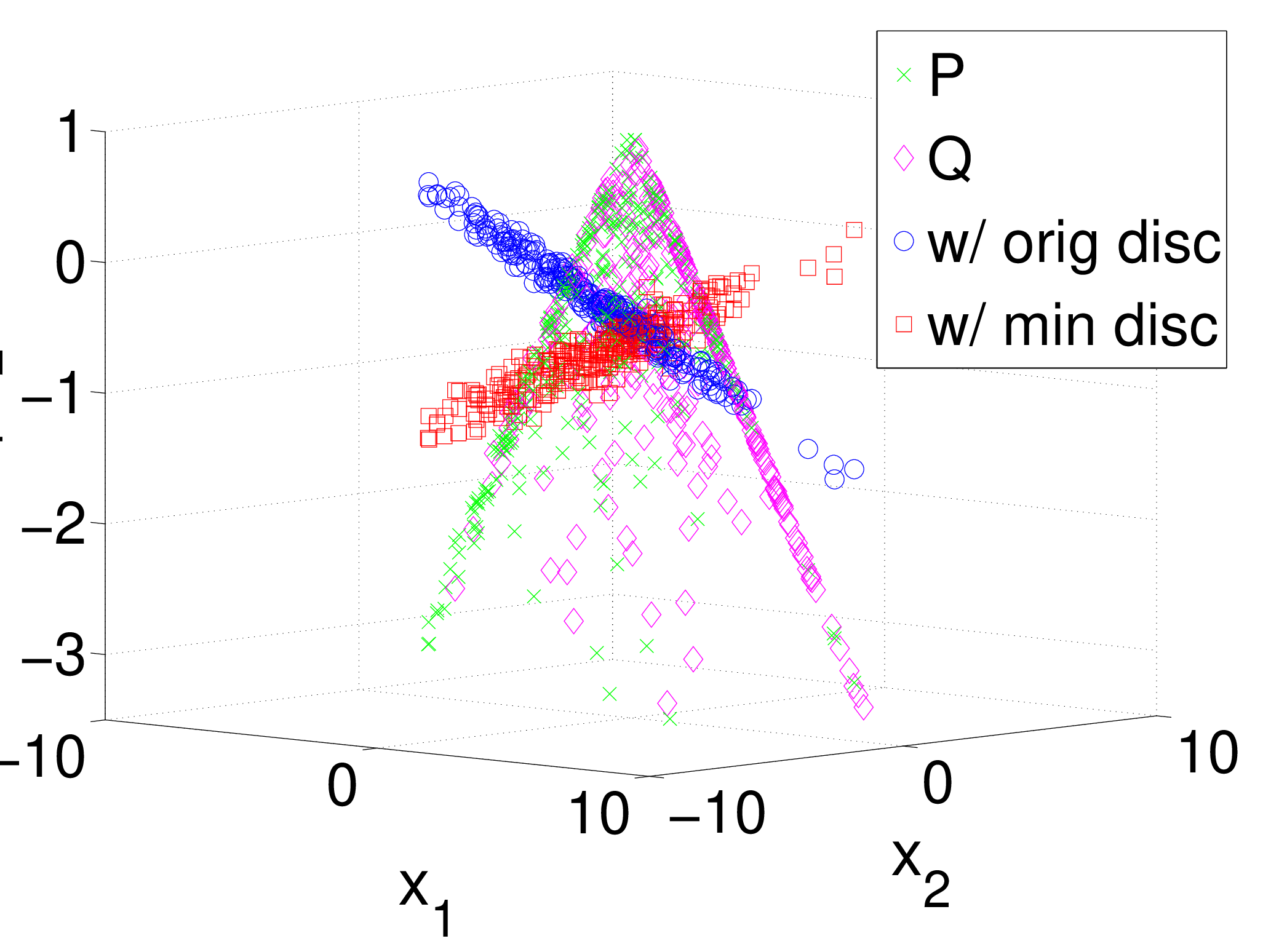}
& \ipsfig{.14}{figure=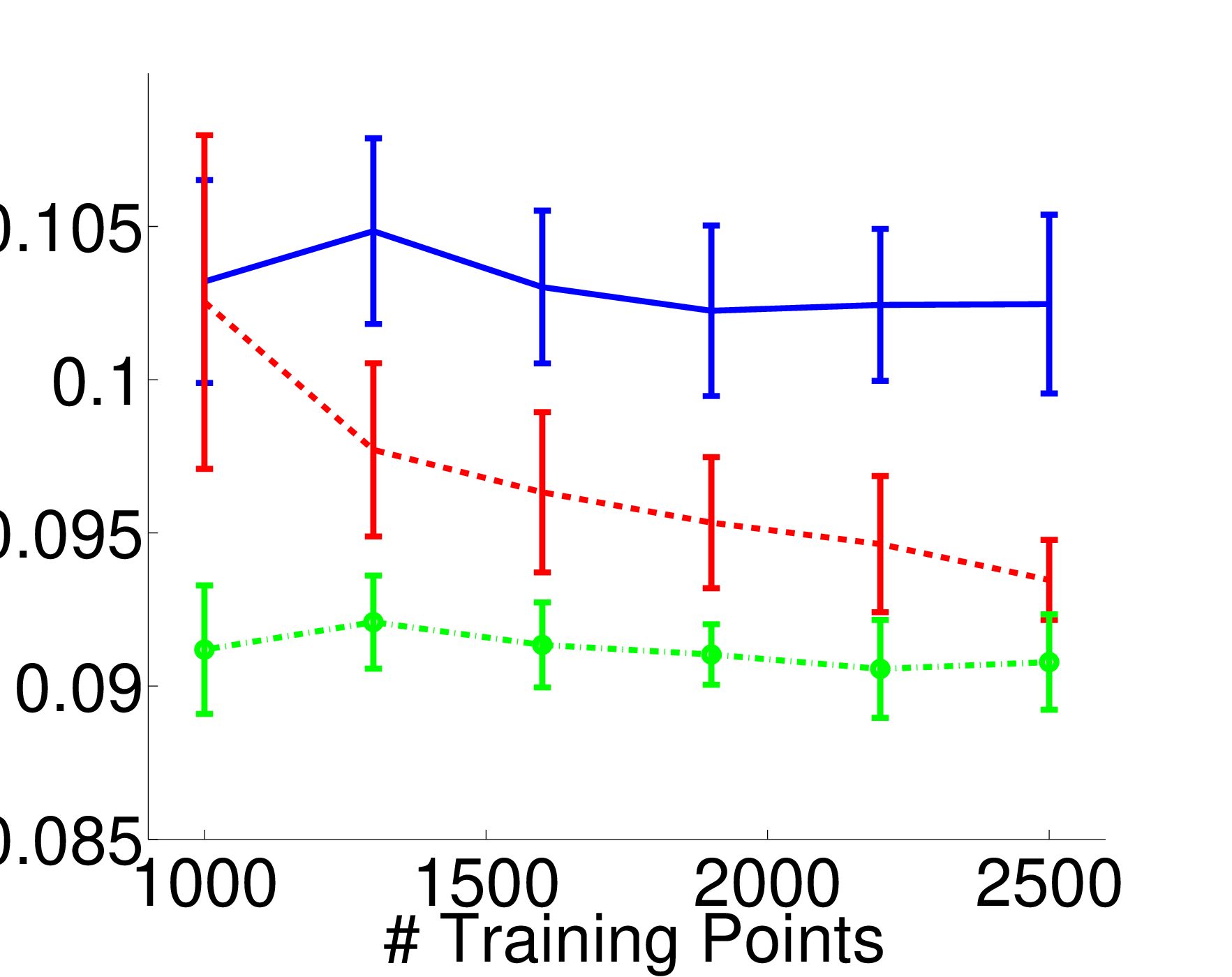}\\
(a) & (b)
\end{tabular}
\end{center}
\vspace{-.5cm}
\caption{(a) An $(x_1,x_2,y)$ plot of $\h Q$ (magenta), $\h P$
(green), weighted (red) and unweighted (blue) hypothesis. (b)
Comparison of mean-squared error for the hypothesis trained on $\h Q$
(top), trained on $\h Q'$ (middle) and on $\h P$ (bottom)
over, a varying number of training points.}
\label{fig:ex_sdp}
\vspace{-.5cm}
\end{figure}

Figures~\ref{fig:ex_sdp}(a)-(b) show the application of the SDP
derived in (\ref{eq:sdp}) to determining the distribution minimizing
the empirical discrepancy for ridge regression.  In
Figure~\ref{fig:ex_sdp}(a), the distributions $Q$ and $P$ are
Gaussians centered at $(\sqrt{2},\sqrt{2})$ and $(-\sqrt{2},
-\sqrt{2})$, both with covariance matrix $2 \I$.  The target function
is $f(x_1, x_2) = (1 - |x_1|) + (1 - |x_2|)$, thus the optimal linear
prediction derived from $Q$ has a negative slope, while the optimal
prediction with respect to the target distribution $P$ in fact has a
positive slope.  Figure~\ref{fig:ex_sdp}(b) shows the performance of
ridge regression when the example is extended to 16-dimensions, before
and after minimizing the discrepancy. In this higher-dimension setting
and even with several thousand points, using ({\small
  http://sedumi.ie.lehigh.edu/}), our SDP problem could be solved in
about 15s using a single 3GHz processor with 2GB RAM. The SDP
algorithm yields distribution weights that decrease the discrepancy
and assist ridge regression in selecting a more appropriate hypothesis
for the target distribution.

\section{Conclusion}
\label{sec:conc}

We presented an extensive theoretical and an algorithmic analysis of
domain adaptation. Our analysis and algorithms are widely applicable
and can benefit a variety of adaptation tasks. More efficient versions
of these algorithms, in some instances efficient approximations,
should further extend the applicability of our techniques to
large-scale adaptation problems.

\appendix

\section{Proof of Theorem~\ref{th:rademacher}}

\begin{proof}
  Let $\Phi(\S)$ be defined by $\Phi(\S) = \sup_{h \in H} R(h) - \h
  R(h)$. Changing a point of $\S$ affects $\Phi(\S)$ by at most
  $1/m$. Thus, by McDiarmid's inequality applied to $\Phi(\S)$, for any
  $\delta > 0$, with probability at least $1 - \frac{\delta}{2}$, the
  following holds for all $h \in H$:
\begin{equation}
\label{eq:mcd1}
\Phi(\S) \leq \E_{\S \sim D}[\Phi(\S)] + \sqrt{\frac{\log \frac{2}{\delta}}{2m}}.
\end{equation}
$\E_{\S \sim D}[\Phi(\S)]$ can be bounded in terms of the empirical
Rade-macher complexity as follows:
\begin{multline*}
\E_{\S}[\Phi(\S)] = \E_\S \big[ \sup_{h \in H} \E_{\S'}[R_{\S'}(h)] - R_\S(h) \big]\\
\begin{aligned}
& = \E_\S \big[ \sup_{h \in H} \E_{\S'}[R_{\S'}(h) - R_\S(h)] \big]\\
& \leq \E_{\S, \S'} \big[ \sup_{h \in H} R_{\S'}(h) - R_\S(h) \big]\\
& = \E_{\S, \S'} \big[ \sup_{h \in H} \frac{1}{m}
\sum_{i = 1}^m (h(x'_i) - h(x_i)) \big]\\
& = \E_{\sigma, \S, \S'} \big[ \sup_{h \in H} \frac{1}{m}
\sum_{i = 1}^m \sigma_i (h(x'_i) - h(x_i)) \big]\\
& \leq \E_{\sigma, \S'} \big[ \sup_{h \in H} \frac{1}{m}
\sum_{i = 1}^m \sigma_i h(x'_i) \big]
+ \E_{\sigma, \S} \big[ \sup_{h \in H} \frac{1}{m}
\sum_{i = 1}^m -\sigma_i h(x_i) \big]\\
& = 2 \E_{\sigma, \S} \big[ \sup_{h \in H} \frac{1}{m}
\sum_{i = 1}^m \sigma_i h(x_i) \big]
\leq 2 \E_{\sigma, \S} \big[ \sup_{h \in H} \big| \frac{1}{m}
\sum_{i = 1}^m \sigma_i h(x_i) \big| \big] \\
\label{eq:27}
& = \R_m(H).
\end{aligned}
\end{multline*}
Changing a point of $\S$ affects $\R_m(H)$ by at most $2/m$. Thus, by
McDiarmid's inequality applied to $\R_m(H)$, with probability at least
$1 - \delta/2$, the following holds:
\begin{equation}
\R_m(H) \leq \h\R_\S(H) + \sqrt{\frac{2 \log \frac{2}{\delta}}{m}}.
\end{equation}
Combining this inequality with Inequality~(\ref{eq:mcd1}) and the
bound on $\E_{\S}[\Phi(\S)]$ above yields directly the statement of
the theorem.
\end{proof}

\section{Rademacher Classification Bound}

\begin{theorem}[Rademacher Classification Bound]
\label{th:rademacher_classification}
Let $H$ be a family of functions mapping $X$ to $\set{0, 1}$ and let
$L_{01}$ denote the 0-1 loss. Let $Q$ be a distribution over
$X$. Then, for any $\delta > 0$, with probability at least $1 -
\delta$, the following inequality holds for all samples $\S$ of size
$m$ drawn according to $Q$:
\begin{equation}
{\L_{01}}_Q(h, h_Q^*) \leq {\L_{01}}_{\h Q}(h, h_Q^*) + \h \R_\S(H)/2 + \sqrt{\frac{\log \frac{1}{\delta}}{2m}}.
\end{equation}
\end{theorem}

\section{Discrepancy Minimization with Kernels and $L_2$ loss}

Here, we show how to generalize the results of Section~\ref{sec:l2} to
the high-dimensional case where $H$ is the reproducing kernel Hilbert
space associated to a positive definite symmetric (PDS) kernel $K$.

\begin{proposition}
\label{prop:sdp_kernel}
Let $K$ be a PDS kernel and let $H$ denote the reproducing kernel
Hilbert space associated to $K$. Then, for any $\h Q$ and $\h P$, the
problem of determining the discrepancy minimizing distribution $\h Q'$
for the square loss can be cast an SDP depending only on the Gram
matrix of the kernel function $K$ and solved in time $O(m_0^2 (m_0 +
n_0)^{2.5} + n_0(m_0 + n_0)^2)$.
\end{proposition}

\begin{proof}
  Let $\Phi\colon X \to H$ be a feature mapping associated with $K$.
  Let $p_0 = m_0 + n_0$. Here, we denote by $s_1, \ldots, s_{m_0}$ the
  elements of $S_Q$ and by $s_{m_0 + 1}, \ldots, s_{p_0}$ the
  element of $S_P$.  We also define $z_i = \h Q'(s_i)$ for $i \in [1,
  m_0]$, and for convenience $z_i = 0$ for $i \in [m_0 + 1, m_0 +
  n_0]$.  Then, by Proposition~\ref{prop:sdp}, the problem of finding
  the optimal distribution $\h Q'$ is equivalent to
\begin{equation}
\label{eq:48}
  \min_{\substack{\| \z \|_1 = 1\\ \z \geq 0}} \set{\lambda_{\max}(\M(\z)), \lambda_{\max}(-\M(\z))},
\end{equation}
where $\M(\z) = \sum_{i = 1}^{p_0} (\h P(s_i) - z_i) \Phi(s_i)
\Phi(s_i)^\top$. Let $\P$ denote the matrix in $\Rset^{N \times p_0}$
whose columns are the vectors $\Phi(s_1), \ldots, \Phi(s_{m_0 +
  n_0})$. Then, observe that $\M(\z)$ can be rewritten as
\begin{equation}
\M(\z) = \P \A \P^\top,
\end{equation}
where $\A$ is the diagonal matrix
\begin{equation}
\A = \diag(\h P(s_1) - z_1, \ldots, \h P(s_{p_0}) - z_{p_0}).
\end{equation}
Fix $\z$. There exists $t_0 \in \Rset$ such that, for all
$t \geq t_0$, $\B = \A + t \I$ is a positive
definite symmetric matrix. For any such $t$, let $\N'(\z)$
denote
\begin{equation}
\N'(\z) = \P \B \P^\top.
\end{equation}
Since $\B$ is positive definite, there exists a diagonal matrix
$\B^{1/2} \!\in\! \Rset^{p_0 \times p_0}$ such that $\B = \B^{1/2}
\B^{1/2}$. Thus, we can write $\N'(\z)$ as $\N'(\z) = \Y \Y^\top$ with
$\Y = \P \B^{1/2}$. $\Y \Y^\top$ and $\Y^\top \Y$ have the same
characteristic polynomial modulo multiplication by $X^{N - p_0}$.
Thus, since $\P^\top \P = \K$, the Gram matrix of kernel $K$ for the
sample $S$, $\N'(\z)$ has the same same characteristic polynomial
modulo multiplication by $X^{N - p_0}$ as
\begin{equation}
\N''(\z) = \Y \Y^\top = \B^{1/2} \K \B^{1/2}.
\end{equation}
Now, $\N''(\z)$ can be rewritten as $\N''(\z) = \Z \Z^\top$ with $\Z =
\B^{1/2} \K^{1/2}$. Using the fact that $\Z \Z^\top$ and $\Z^\top \Z$
have the same characteristic polynomial, this shows that $\N'(\z)$ has
the same characteristic polynomial modulo multiplication by $X^{N -
  p_0}$ as
\begin{equation}
\N'''(\z) = \K^{1/2} \B \K^{1/2}.
\end{equation}
Thus, assuming without loss of generality that $N > p_0$, the
following equality between polynomials in $X$ holds for all $t
\geq t_0$:
\begin{equation}
\det(X\I - \P \B \P^\top) = X^{N - p_0}\det(X\I - \K^{1/2} \B \K^{1/2}).
\end{equation}
Both determinants are also polynomials in $t$. Thus, for every fixed
value of $X$, this is an equality between two polynomials in $t$ for
all $t \geq t_0$. Thus, the equality holds for all $t$, in particular
for $t = 0$, which implies that $\M(\z) = \P \A \P^\top$ has the same
non-zero eigenvalues as $\M'(\z) = \K^{1/2} \A \K^{1/2}$.  Thus,
problem (\ref{eq:48}) is equivalent to
\begin{equation}
\label{eq:54}
  \min_{\substack{\| \z \|_1 = 1\\ \z \geq 0}} \set{\lambda_{\max}(\M'(\z)), \lambda_{\max}(-\M'(\z))}.
\end{equation}
Let $\A_0$ denote the diagonal matrix
\begin{equation}
\A_0 = \diag(P(s_1), \ldots, P(s_{p_0})),
\end{equation}
and for $i \in [1, m_0]$, let $\I_i \in \Rset^{p_0 \times p_0}$ denote
the diagonal matrix whose diagonal entries are all zero except from
the $i$th one which equals one. Then,
\begin{equation}
\M'(\z) = \M'_0 - \sum_{i = 1}^{m_0} z_i \M'_i
\end{equation}
with $\M'_0 = \K^{1/2} \A_0 \K^{1/2}$ and $\M'_i = \K^{1/2} \I_i
\K^{1/2}$ for $i \in [1, m_0]$.  Thus, $\M'(\z)$ is an affine
function of $\z$ and problem (\ref{eq:54}) is a convex optimization
problem that can be cast as an SDP, as described in
Section~\ref{sec:l2}, in terms of the Gram matrix $\K$ of the kernel
function $K$.
\end{proof}

\section{Standard Form of SDP Problem}
In this section we explicitly formulate both the \emph{inequality} and
\emph{standard} form of the semidefinite program presented in equation
(\ref{eq:sdp}).
First we write the inequality form:
\begin{align}
\min_{\z, \lambda} & \quad \lambda\\
\text{subject to} 
& \quad \sum_i z_i \M_i - \lambda \I  \preceq \M_0 \\
& \quad \sum_i -z_i \M_i - \lambda \I  \preceq -\M_0 \\
& \quad \sum_i z_i = 1 \\
& \quad -z_i \leq 0, \forall i.
\end{align}
Note that the several linear inequalities can be written
as a single linear matrix inequality (LMI) by using a large block
diagonal matrix. The dual of this problem is then in standard form:
\begin{align}
\max_{\A, \B, \ggamma, \alpha} & \quad -\tr(\A \M_0) + \tr(\B
\M_0) - \alpha \\
\text{subject to} 
& \quad \tr(\A \M_i) - \tr(\B \M_i) + \alpha - \gamma_i = 0, \forall i
\\
& \quad -\tr(\A \I) - \tr(\B \I) = -1\\
& \quad \A \succeq 0, ~\B \succeq 0, ~\ggamma \succeq 0
\end{align}
The standard form problem can be straight-forwardly presented to a
standard solver, such as SeDuMi.  The variable $z_i$ can be retrieved
as the Lagrange multiplier for $i$th equality constraint.

\bibliographystyle{mlapa}
{\small
\bibliography{nadap}
}
\end{document}